\documentclass[]{article} 
\usepackage{geometry}
\usepackage{url}

\usepackage[utf8]{inputenc} 
\usepackage[T1]{fontenc}    
\usepackage{lmodern}
\usepackage{booktabs}       
\usepackage{nicefrac}       
\usepackage{microtype}      

\usepackage{setspace}
\usepackage{makecell}
\usepackage{amsmath,amsthm,amsfonts,amssymb}
\usepackage{bm}
\usepackage{comment}
\usepackage{tablefootnote}
\usepackage{multirow}
\usepackage{hhline}
\usepackage{graphicx}
\usepackage{subfigure}
\usepackage{algorithm,algorithmic}
\usepackage[font=footnotesize,labelfont=bf]{caption}
\usepackage{url}
\usepackage{authblk}
\usepackage{hyperref}
\usepackage{hhline}
\newcommand{\bc}{\mbox{\textit{c}}}
\newcommand{\bbc}{\mbox{\textit{\large c}}}
\newtheorem{definition}{Definition}
\newtheorem{theorem}{Theorem}
\newtheorem{remark}{Remark}
\newtheorem{lemma}[theorem]{Lemma}

\title{Multi-Complementary and Unlabeled Learning for Arbitrary Losses and Models}

\author[1]{Yuzhou Cao}
\author[1]{Shuqi Liu}
\author[1]{Yitian Xu}
\affil[1]{College of Science, China Agricultural University}

\date{}

\begin{document}
\maketitle
\begin{abstract}
A weakly-supervised learning framework named as \textit{complementary-label learning} has been proposed recently, where each sample is equipped with a single complementary label that denotes one of the classes the sample does \textit{not} belong to. However, the existing complementary-label learning methods cannot learn from the easily accessible unlabeled samples and samples with multiple complementary labels, which are more informative. In this paper, to remove these limitations, we propose the novel \textit{multi-complementary and unlabeled learning} framework that allows unbiased estimation of classification risk from samples with any number of complementary labels and unlabeled samples, for arbitrary loss functions and models. We first give an unbiased estimator of the classification risk from samples with multiple complementary labels, and then further improve the estimator by incorporating unlabeled samples into the risk formulation. The estimation error bounds show that the proposed methods are in the optimal parametric convergence rate. Finally, the experiments on both linear and deep models show the effectiveness of our methods.
\end{abstract}

\section{Introduction}
The ordinary supervised classification problems require that each training sample should be equipped with an exact label that denotes the class the sample belongs to. However, the preparation of massive exactly labeled data is usually laborious and unrealistic in practical. Therefore, a lot of studies on learning from weak supervision have been made to tackle this problem in different scenarios, e.g. semi-supervised learning \cite{DBLP:books/mit/06/CSZ2006,DBLP:conf/icml/SakaiPNS17,DBLP:conf/ijcai/ZhangZ18,SSL1,SSL2}, partial label learning \cite{DBLP:journals/jmlr/CourST11,DBLP:journals/tkde/ZhangYT17}, and positive-unlabeled learning \cite{DBLP:conf/kdd/ElkanN08,DBLP:conf/icml/PlessisNS15,DBLP:journals/pami/SansoneNZ19,DBLP:conf/icml/HsiehNS19}. Recently, another weakly-supervised learning scenario called \textit{complementary-label learning} (CLL) has been proposed. In the CLL setting, each ordinary label is substituted with the complementary label, which denotes one of the classes that a training sample does \textit{not} belong to. It is obvious that the preparation of complementarily labeled data is much more labor-saving than that of ordinarily labeled data.

The complementary-label learning problem has been investigated in previous studies \cite{DBLP:conf/nips/IshidaNHS17,DBLP:conf/eccv/YuLGT18,DBLP:conf/icml/IshidaNMS19}. In these works, different risk estimators were proposed to recover classification risk only from complementarily labeled data under the empirical risk minimization (ERM) framework. In \cite{DBLP:conf/nips/IshidaNHS17} and \cite{DBLP:conf/eccv/YuLGT18}, the proposed risk estimators had restrictions on loss functions and unbiasedness respectively. \cite{DBLP:conf/icml/IshidaNMS19} overcame the shortcomings by giving an unbiased risk estimator without any restriction on models and loss functions while guaranteeing the superior performance in terms of classification accuracy over the previous two methods.

It is noticeable that in these works, each training sample was given only a single complementary label. However, in quite a few cases, the training samples can be multi-complementarily labeled, namely each training sample is equipped with multiple complementary labels. For example, in the stage of data annotation, an annotator who has no idea of a training sample's exact label may be able to recognize multiple classes that the sample does not belong to, which results in a sample with multiple complementary labels. In crowdsourcing scenario \cite{Crowdsourcing,CS}, the quality of crowdsourcing label is especially crucial \cite{DBLP:journals/chinaf/WangZ15a}. Instead of being ordinarily labeled, a sample can be complementarily labeled to alleviate the effect of low-quality noisy crowdsourcing labels. Since a sample can be complementarily labeled by different crowdworkers, each training sample may have more than one complementary label. Moreover, compared with the single-complementary-label setting in previous CLL studies, the samples with multiple complementary labels are more informative. To sum up, a framework for learning from data with arbitrary\footnote{ `Arbitrary' means the samples can be equipped with different numbers of complementary labels.} number of complementary labels is in demand.

Furthermore, the information concealed in the easily accessible unlabeled data proved to be helpful in many other weakly-supervised learning scenario both theoretically and practically \cite{shai,incomplete,partial_unlabeled}. Therefore, it is promising to further enhance the capability of CLL framework by incorporating the unlabeled data.

In this paper, we study the \textit{multi-complementary label and unlabeled learning} (MCUL) problem, where both multi-complementarily labeled data and unlabeled data are leveraged to obtain better classifiers. In our method, we propose a novel unbiased risk estimator for MCUL problem with no limitation on loss functions and models. By using a mild assumption, we first derive the risk estimator for \textit{multi-complementary label learning} (MCL) problem. Then we further utilize the unlabeled data to construct the risk estimator for MCUL problem. With no more assumption on loss functions and models, we show that the estimation error bounds of MCL and MCUL are in optimal parametric convergence rate \cite{vapnik}. The effectiveness of the proposed MCUL is demonstrated through experiments on both linear and deep models.

The main contributions are summarized as follows:
\begin{itemize}
\item We propose the novel MCL framework that allows unbiased estimation of the classification risk only from samples with arbitrary number of complementary labels and can be applied on \textit{arbitrary losses and models}.

\item We further propose the MCUL framework to utilize the unlabeled samples, which are neglected in previous studies on CLL\cite{DBLP:conf/nips/IshidaNHS17,DBLP:conf/eccv/YuLGT18,DBLP:conf/icml/IshidaNMS19} and validate the benefits of the incorporation of unlabeled samples both experimentally and theoretically.

\item The previous CLL framework and ordinary classification problems are proven to be special cases of the MCUL framework, which  shows the comprehensiveness of the MCUL framework as a weakly-supervised leaning framework.
\end{itemize}

This rest of this paper is organized as follows. We give the review of complementary-label learning in Section 2. The MCL and MCUL frameworks are proposed in Section 3. Moreover, we analyse the estimation error bounds of the proposed methods in Section 4 and discuss the helpfulness of integrating the class-prior information in Section 5. Finally, we give the experimental results of our frameworks on both linear and deep models in Section 6 and conclude the paper in Section 7. The detailed proof is shown in the appendix.
\section{Review of Complementary-Label Learning}
To begin with, we first show the classification risk of learning from ordinary labels and then review how the previous risk estimators of learning from complementarily labeled samples recover the classification risk under the ERM framework.
\subsection{Ordinary Classification Problem}
Let's denote the feature space with $\mathcal{X}\in\mathbb{R}^{d}$ and $\mathcal{Y}=\{1,2,\ldots,K\}$ is the label space. The training samples are drawn independently and identically from the unknown distribution $\mathcal{D}$, which is the joint distribution over $\mathcal{X}\times\mathcal{Y}$ with density $p(\bm{x},y)$. Then the critical work is to find a decision function $\bm{g}:\mathcal{X}\rightarrow\mathbb{R}^{K}$ that minimizes the classification risk with loss function $\ell:\mathcal{X}\times\mathcal{Y}\rightarrow\mathbb{R}^{+}$:
\begin{equation}
\label{E1}
R(g):=\mathbb{E}_{p(\bm{x},y)}[\ell(\bm{g}(\bm{x}),y)].
\end{equation}

Since the density $p(\bm{x},y)$ is unknown, the classification risk is approximated by the empirical risk:
\begin{equation}
\label{ordinary}
\hat{R}(g):=\frac{1}{n}\sum_{i=1}^{n}\ell(\bm{g}(\bm{x}_{i}),y_{i}).
\end{equation}
\subsection{Complementary-Label Learning}
In the CLL setting, each sample is equipped with a complementary label. The complementarily labeled data $\{(\bm{x}_i,\overline{y}_{i})\}_{i=1}^{n}$ are sampled independently and identically from a joint distribution with density $\overline{p}(\bm{x},\overline{y})$.

In \cite{DBLP:conf/nips/IshidaNHS17}, an assumption on density $\overline{p}(\bm{x},\overline{y})$ was made:
\begin{equation}
\label{E2}
\overline{p}(\bm{x},\overline{y})=\frac{1}{K-1}\sum_{y\neq\overline{y}}p(\bm{x},y).
\end{equation}
Under this assumption, \cite{DBLP:conf/nips/IshidaNHS17} proved that classification risk (\ref{E1}) can be recovered by an unbiased estimator only from complementarily labeled data. However, the loss functions are restricted to one-versus-all and pairwise comparison multi-class loss functions \cite{TZhang}. Moreover, the binary loss functions $\ell^{'}(z): \mathbb{R}\rightarrow\mathbb{R}^{+}$ used in the two multi-class loss functions are required to fulfill symmetric condition: $\ell^{'}(z)\!+\!\ell^{'}(-z)=1$. Obviously, the popular softmax cross-entropy loss and all the other convex loss functions do not meet these conditions. Since the softmax cross-entropy loss is widely used in deep learning, this requirement will be a serious limitation for the application of state-of-the-art deep models.

To make deep models available, \cite{DBLP:conf/eccv/YuLGT18} proposed another risk estimator limited to softmax cross-entropy loss. Though the risk estimator is not necessarily unbiased, the method is ensured to identify the optimal classifier that minimizes classification risk (\ref{E1}) by minimizing its learning object. The method also introduces bias into the choice of complementary labels. However, in the stages of bias estimation, ordinarily labeled data are required. The severe requirement might not align with the motivation of complementary-label learning.

The limitations above were removed in \cite{DBLP:conf/icml/IshidaNMS19}. An unbiased risk estimator with only complementarily labeled data was deduced by taking a different approach than \cite{DBLP:conf/nips/IshidaNHS17}. With the same assumption (\ref{E2}) adopted, the risk formulation is valid for arbitrary losses and models. Experiments on both linear and deep models showed the superiority of the estimator in \cite{DBLP:conf/icml/IshidaNMS19} than those in previous works \cite{DBLP:conf/nips/IshidaNHS17,DBLP:conf/eccv/YuLGT18}. Nevertheless, the estimator is still confined within single-complementary-label setting, where each sample is given merely one complementary label. The unlabeled data are also neglected in previous CLL studies, which prevents the CLL from being a more general framework.

\begin{figure*}[t]
    \centerline{\includegraphics[width = 20cm, height = 7.5cm,trim=10 140 160 80,clip]{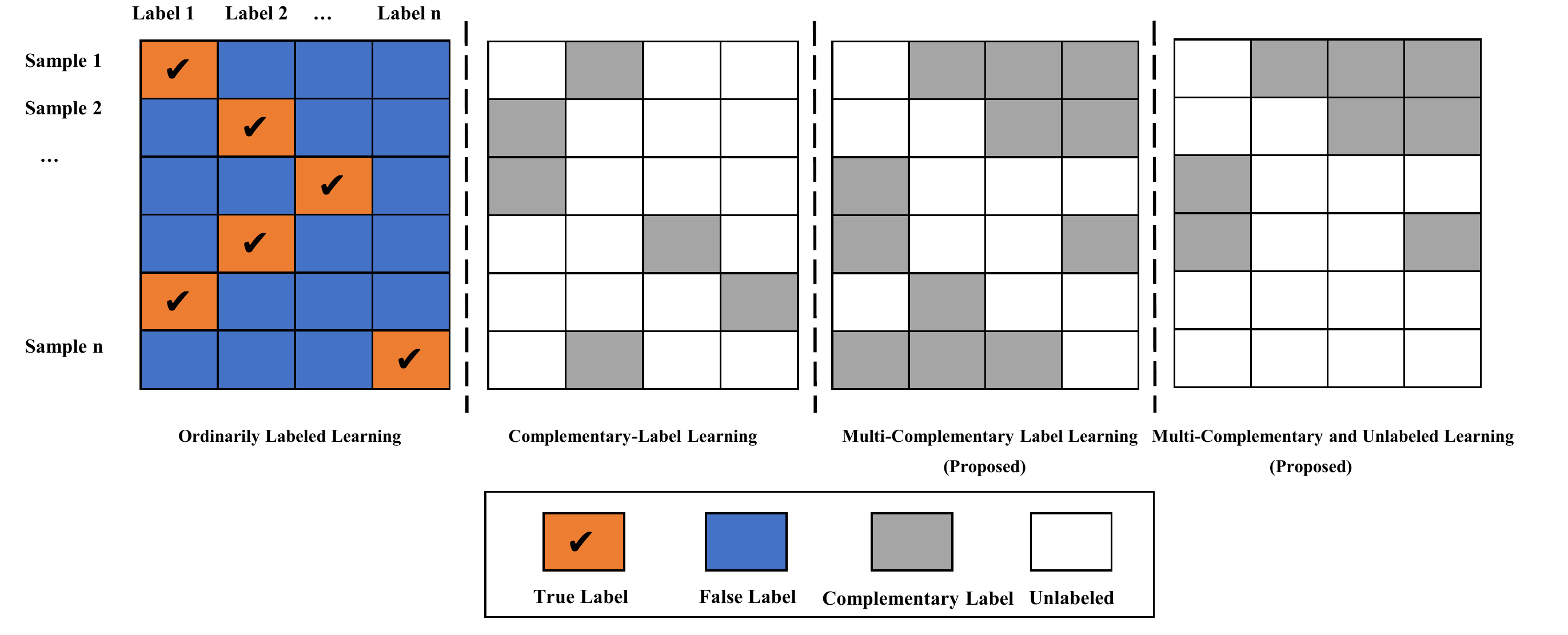}}
    \caption{The demonstration of previous CCL setting and the proposed MCL\&MCUL settings.}
    \label{Comp}
\end{figure*}

The MCUL framework proposed in this paper further enables learning from both multi-complementarily labeled samples and unlabeled samples. Figure \ref{Comp} describes the differences between the previous CLL setting and the proposed MCUL setting.
\section{Proposed Frameworks}
In this section, we propose our framework to enable unbiased estimation of classification risk from both multi-complementarily labeled data and unlabeled data.

We first prove that the classification risk can be recovered from multi-complementarily labeled data under a mild assumption by employing the risk rewrite technique \cite{UU}. Then we further present the risk formulation of MCUL and show the estimation error bounds of the two methods.

{\em \textbf{Notations and Settings:}} Denote by $\overline{\mathcal{Y}}\!=\!\{1,\ldots,K\}$ the complementary label space. $\mathcal{\overline{Y}}_{c}$ is the collection of all the possible combinations of $c$ different complementary labels, e.g. $\overline{Y}_{c}=\{1,\ldots,c\}\in\mathcal{\overline{Y}}_{c}$. $\overline{Y}_{c}$ is referred to as complementary-label set in the following sections. Suppose the training samples are sampled as follows:
\begin{align}
    &\mathcal{S}_{u}:=\{\bm{x}^{u}_{i}\}_{i=1}^{n_{u}}\mathop{\sim}\limits^{i.i.d.} p(\bm{x}),\nonumber\\
    &\mathcal{S}_{c}:=\{(\bm{x}^{c}_{i},\overline{Y}^{c}_{i})\}_{i=1}^{n_{c}}\mathop{\sim}\limits^{i.i.d.} \overline{p}_{c}(\bm{x},\overline{Y}),~~c=1,\ldots,K\!-\!1.\nonumber
\end{align}
where $p(\bm{x})$ is the marginal density and $\overline{p}_{c}(\bm{x},\overline{Y})$ is the density on $\mathcal{X}\times\overline{\mathcal{Y}}_{c}$. $\mathcal{S}_{u}$ are the unlabeled data and $\mathcal{S}_{c}$ are the multi-complementarily labeled data with complementary-label sets of size $c$. The size of complementary-label set $\overline{Y}$ is denoted by $|\overline{Y}|$ and $|\mathcal{S}_{c}|=n_{c},~c=1,\ldots,K-1$.
\subsection{Multi-Complementary Label Learning (MCL)}
In this section, we give an account of the risk minimization framework of multi-complementary label learning.

As in the previous works, we first make assumptions on the relation between density $\overline{p}_{c}(\bm{x},\overline{Y})$ and $p(\bm{x},y)$.
\begin{align}
\label{E3}
             \overline{p}_{c}(\bm{x},\overline{Y})=\dfrac{1}{\binom{K\!-1}{\bc}}\sum\limits_{y\notin\overline{Y}}p(\bm{x},y).
\end{align}

The assumption implies each combination of $c$ complementary labels are selected uniformly, which is a mild generalized assumption of those in the previous works \cite{DBLP:conf/nips/IshidaNHS17,DBLP:conf/eccv/YuLGT18,DBLP:conf/icml/IshidaNMS19}. Under this assumption, we prove that the \textit{multi-complementary loss} allows unbiased estimation of classification risk (\ref{E1}) from samples with complementary-label sets:
\begin{lemma}
\label{T1}Suppose the density $\overline{p}_{c}(\bm{x},\overline{Y})$ and $p(\bm{x},y)$ follow the assumption (\ref{E3}). For \textbf{any} loss function $\ell$ and decision function $\bm{g}$, the classification risk (\ref{E1}) is equal to the risk formulation below:
\begin{equation}
\label{E4}
R_{c}(\bm{g})=\mathbb{E}_{\overline{p}_{c}(\bm{x},\overline{Y})}[\overline{\ell}(\bm{g}(\bm{x}),\overline{Y})].
\end{equation}
where $\overline{\ell}$ is the multi-complementary loss:
\begin{equation}
\overline{\ell}(\bm{x},\overline{Y}):=\sum\limits_{y=1}^{K}\ell(\bm{g}(\bm{x}),y)-\dfrac{K\!-\!1}{|\overline{Y}|}\sum\limits_{y\in\overline{Y}}\ell(\bm{g}(\bm{x}),y).
\end{equation}
\end{lemma}
The proof can be found in the Appendix \ref{AA}. For ease of notation, we the following notation for cumulative loss $\sum\limits_{y=1}^{K}\ell(\bm{g}(\bm{x}),y)$:
\begin{align}
\mathcal{L}(\bm{g}(\bm{x})):=\sum\limits_{y=1}^{K}\ell(\bm{g}(\bm{x}),y).
\end{align}
Due to the notation, we can further rewrite the multi-complementary loss into the form below:
\begin{align}
\label{E5}
\overline{\ell}(\bm{x},\overline{Y}):=\mathcal{L}(\bm{g}(\bm{x}))-\dfrac{K\!-\!1}{|\overline{Y}|}\sum\limits_{y\in\overline{Y}}\ell(\bm{g}(\bm{x}),y).
\end{align}
Notice that the cumulative loss $\mathcal{L}(\bm{g}(\bm{x}))=\sum\limits_{y=1}^{K}\ell(\bm{g}(\bm{x}),y)$ is obtained by summing up the loss of the prediction $\bm{g}(\bm{x})$ \textit{w.r.t.} all the potential labels $y\in\{1,\ldots,K\}$, so it only relies on the sample $\bm{x}$ and the classifier $\bm{g}(\cdot)$. As a result, the label information used in the calculation of multi-complementary loss (\ref{E5}) is only complementary-label set $\overline{Y}$. Therefore, the multi-complementary label learning setting totally gets rid of the dependence on true labels.

The risk formulation in Lemma \ref{T1} shows that the classification risk can be recovered only from samples with complementary-label sets of fixed size $c$. However, the complementary-label sets of samples are \textit{not} necessarily limited to a certain size in reality. To completely remove the limitation on the size of complementary-label set, we consider the convex combination of $R_{c}(\bm{g})$ called \textit{multi-complementary risk}.
\begin{definition}\label{D1}(Multi-Complementary Risk)~For any decision function $\bm{g}$, its MCL risk is defined as:
\begin{equation}
\label{tot1}
    R_{\mbox{\tiny{\rm{MCL}}}}(\bm{g})=\sum_{c=1}^{K-1}\alpha_{c}R_{c}(\bm{g}),
\end{equation}
where $\bm{\alpha}$ is any vector in $\{\bm{\alpha}\big|\sum\limits_{c=1}^{K-1}\alpha_{c}=1,~\bm{\alpha}\succeq\bm{0}\}$.
\end{definition}
\begin{theorem}
The MCL risk is equal to classification risk (\ref{E1}):
\begin{equation}
    R_{\mbox{\tiny{\rm{MCL}}}}(\bm{g})=R(\bm{g}).
\end{equation}
\end{theorem}
\begin{proof}
Due to Lemma \ref{T1}, we can get $R_{c}(\bm{g})=R(\bm{g})$. Then the following equations holds:$$R_{\mbox{\tiny{MCL}}}(\bm{g})=\sum\limits_{c=1}^{K-1}\alpha_{c}R_{c}(\bm{g})=\sum\limits_{c=1}^{K-1}\alpha_{c}R(\bm{g})=R(\bm{g}).$$
\end{proof}
The empirical MCL risk is as below:
\begin{align}
\label{unbiased}
\hat{R}_{\mbox{\tiny{MCL}}}(\bm{g})
=\sum\limits_{c=1}^{K-1}\dfrac{\alpha_{c}}{n_{c}}\sum\limits_{i=1}^{n_{c}}\left(\mathcal{L}(\bm{g}(\bm{x}_{i}^{c}))-\dfrac{K\!-\!1}{\bbc}\sum\limits_{y\in\overline{Y}_{i}^{c}}\ell(\bm{g}(\bm{x}_{i}^{c}),y)\right).
\end{align}
Then the following work is to find the minimizer $\hat{\bm{g}}_{\mbox{\tiny{MCL}}}$ of empirical MCL risk:
\begin{equation}
    \hat{\bm{g}}_{\mbox{\tiny{MCL}}}=\mathop{\mbox{arg~min}}\limits_{\bm{g}\in\mathcal{G}^{K}}\hat{R}_{\mbox{\tiny{MCL}}}(\bm{g}).
\end{equation}
where $\mathcal{G}=\{g(\bm{x})\}$ is a real function class and $\mathcal{G}^{K}=[\mathcal{G}_{i}]^{K}_{i=1}$ is a $K$-dimensional function class.

In (\ref{unbiased}), all the samples are taken into consideration regardless of the size of their complementary-label sets. Since there is no restriction on loss function $\ell$ and classifier $\bm{g}$, any loss and model is available for the multi-complementary learning framework.

\begin{remark}{\rm
There are some special cases in the multi-complementary label learning setting. If $\alpha_{1}=1$, the proposed estimator will reduce to the estimator in single-complementary-label setting \cite{DBLP:conf/icml/IshidaNMS19}. If $\alpha_{K-1}=1$, the proposed estimator will be the same with that in ordinary classification problem (\ref{ordinary}). According to the special cases, the proposed MCL proved to be a comprehensive weakly-supervised learning framework.}
\end{remark}
\subsection{Multi-Complementary and Unlabeled Learning (MCUL)}
To utilize both multi-complementarily labeled data and unlabeled data, we further rewrite the risk formulation and propose the MCUL framework. Based on Lemma \ref{T1}, we can incorporate the unlabeled data to construct an unbiased estimator of classification risk (\ref{E1}):
\begin{lemma}
\label{T2}
The classification risk (\ref{E1}) is equal to the risk formulation below:
\begin{align}
\label{Eu}
R_{c}^{u}(\bm{g})\!&=\!\mathbb{E}_{\overline{p}_{c}(\bm{x},\overline{Y})}\left[(1\!-\!\gamma)\mathcal{L}(\bm{g}(\bm{x}))\!-\!\dfrac{K\!-\!1}{c}\sum\limits_{y\in\overline{Y}}\ell(\bm{g}(\bm{x}),y)\right]\!+\!\gamma\mathbb{E}_{p(\bm{x})}\!\left[\mathcal{L}(\bm{g}(\bm{x}))\right]
\end{align}
where $\gamma\in[0,1]$ is the trade-off coefficient.
\end{lemma}
\begin{proof}
The cumulative loss $\mathcal{L}(\bm{g}(\bm{x}))$ is independent of $\overline{Y}$, and thus:
\begin{align}
\mathbb{E}_{p(\bm{x})}\left[\mathcal{L}(\bm{g}(\bm{x}))\right]=\mathbb{E}_{\overline{p}_{c}(\bm{x},\overline{Y})}\left[\mathcal{L}(\bm{g}(\bm{x}))\right].\nonumber
\end{align}
According to the equation above and Lemma \ref{T1}, we can obtain:$$R_{c}^{u}(\bm{g})=R_{c}(\bm{g})=R(\bm{g}).$$
\end{proof}
In the same manner as in the derivation of (\ref{tot1}), we can derive the \textit{multi-complementary and unlabeled risk}:
\begin{definition}(Multi-Complementary\&Unlabeled Risk)~For any decision function $\bm{g}$, its MCUL risk is defined as:
\label{D2}
\begin{align}
\label{MCUL}
    R_{\mbox{\tiny{{\rm MCUL}}}}(\bm{g})&=\sum\limits_{c=1}^{K-1}\alpha_{c}R_{c}^{u}(\bm{g}).
\end{align}
where $\bm{\alpha}$ is any vector in $\{\bm{\alpha}\big|\sum\limits_{c=1}^{K-1}\alpha_{c}=1,~\bm{\alpha}\succeq\bm{0}\}$.
\end{definition}
When the trade-off coefficient $\gamma$ is set to 0, the MCUL risk is the same with MCL risk (\ref{tot1}). The following Theorem allows unbiased estimation with both unlabeled data and multi-complementarily labeled data.
\begin{theorem}
The MCUL risk is equal to classification risk (\ref{E1}):
    \begin{equation}
        R_{\mbox{\tiny{{\rm MCUL}}}}(\bm{g})=R(\bm{g}).
    \end{equation}
\end{theorem}
The Theorem can be proven in the same way as in Theorem \ref{tot1}.
We can approximate the MCUL risk by the empirical MCUL risk below:
\begin{align}
\label{EMMCUL}
    \hat{R}_{\mbox{\tiny{MCUL}}}(\bm{g})=&\dfrac{\gamma}{n_{u}}\sum\limits_{i=1}^{n_{u}}\mathcal{L}(\bm{g}(\bm{x}^{u}_{i}))+\sum\limits_{c=1}^{K-1}\dfrac{\alpha_{c}(1-\gamma)}{n_{c}}\mathcal{L}(\bm{g}(\bm{x}_{i}^{c}))\nonumber\\&-\sum\limits_{c=1}^{K-1}\dfrac{\alpha_{c}(K\!-\!1)}{c\!\cdot\!n_{c}}\sum\limits_{i=1}^{n_{c}}\sum\limits_{y\in\overline{Y}_{i}^{c}}\ell(\bm{g}(\bm{x}_{i}^{c}),y).
\end{align}
Notice that the unlabeled data are used for construct the estimator of cumulative loss since the calculation of $\mathcal{L}(\bm{g}(\bm{x}))$ does not need any label information. Theorem \ref{T2} shows that this incorporation can still yields an unbiased estimator of classification risk (\ref{E1}).

Then the following work is to find the minimizer $\hat{\bm{g}}_{\mbox{\tiny{MCUL}}}$ of empirical MCUL risk:
\begin{equation}
    \hat{\bm{g}}_{\mbox{\tiny{MCUL}}}=\mathop{\mbox{arg~min}}\limits_{\bm{g}\in\mathcal{G}^{K}}\hat{R}_{\mbox{\tiny{MCUL}}}(\bm{g}).
\end{equation}

Compared with the empirical MCL risk (\ref{unbiased}), the empirical MCUL risk (\ref{MCUL}) further incorporates the unlabeled data into the risk formulation. With the incorporation of easily accessible unlabeled data, the estimation error bound will be tighter, which indicates a better decision function $\bm{g}$. The claim is further validated in the following sections.
\section{Estimation Error Bounds of MCL and MCUL}
In this section, we give the estimation error bounds of the proposed MCL and MCUL frameworks.

Suppose the non-negative loss function $\ell$ does not exceed $C_{\ell}$ on feature space $\mathcal{X}$ and let $L_{\ell}$ be the Lipschitz constant of $\ell$. $\mathfrak{R}_{n}(\mathcal{G})$ is the Rademacher complexity \cite{foundation} of function class $\mathcal{G}$ with sample size of $n$ from $p(\bm{x})$ and we suppose it decays in the rate of $\mathcal{O}(1/\sqrt{n})$. We have the following estimation error bounds, which show the convergence of $\hat{\bm{g}}_{\mbox{\tiny{MCL}}}$ and $\hat{\bm{g}}_{\mbox{\tiny{MCUL}}}$ to the optimal decision function $\bm{g}^{*}=\mathop{\mbox{arg~min}}\limits_{\bm{g}\in\mathcal{G}^{K}}R(\bm{g})$:
\begin{theorem}(Estimation error bound of MCL) For any $\delta>0$, with probability at least $1-\delta$:
\label{MCL_bound}
\begin{align}
R(\hat{\bm{g}}_{\mbox{\tiny{MCL}}})\!-\!R(\bm{g}^{*})&\leq\sum\limits_{c=1}^{K-1}2K\left(\frac{K-1}{\bc}+1\right)\alpha_{c}\nonumber\\&\times\left(2KL_{\ell}\mathfrak{R}_{n_{c}}(\mathcal{G})+C_{\ell}\sqrt{\dfrac{\ln(2K/\delta)}{2n_{c}}}\right).
\end{align}
\end{theorem}
\begin{theorem}(Estimation error bound of MCUL) For any $\delta>0$, with probability at least $1-\delta$:
\label{MCUL_bound}
\begin{align}
R(\hat{\bm{g}}_{\mbox{\tiny{MCUL}}})\!-\!R(\bm{g}^{*})=\mathcal{O}(1/\sqrt{n_{u}}+\sum\limits_{c=1}^{K-1}1/\sqrt{n_{c}}).
\end{align}
\end{theorem}
The proof of the theorems above can be found in the Appendix \ref{AB}.
\begin{remark}{\rm From Theorems \ref{MCL_bound} and \ref{MCUL_bound}, we can learn the estimation error bounds of the proposed methods are in the optimal convergence rate without any additional assumption \cite{vapnik}. Moreover, with increasing number of unlabeled data, the error bound of MCUL will get tighter, which implies the helpfulness of utilizing unlabeled data.}\end{remark}
\section{Integration of Class-Prior Information}
In the previous sections, a sample with $c$ complementary labels is considered to be sampled from the distribution with density $\overline{p}_{c}(\bm{x},\overline{Y})$, which is independent from the class-prior probability $p(|\overline{Y}|=c)$. In practical situations, however, the class-prior may be accessible. For example, in \cite{CR}, the class-prior can be estimated from the given data; the prior is simply approximated by the relative frequency in \cite{DBLP:conf/icml/IshidaNMS19}. \cite{PUPN} proves the helpfulness of integrating the class-prior into learning algorithm. As can be seen, it is promising to further enhance the capability of proposed MCL and MCUL framework by utilizing the class-prior information.

Notice that ${\overline{p}_{c}(\bm{x},\overline{Y})}$ is the conditional density $\overline{p}\left(\bm{x},\overline{Y}\left||\overline{Y}|=c\right.\right)$ in essence. Then the following equation holds:
\begin{align}
\label{cond}
\overline{p}(\bm{x},\overline{Y})=\overline{p}_{c}(\bm{x},\overline{Y})\pi_{c}.
\end{align}
where $\pi_{c}=p(|\overline{Y}|=c)$ and $|\overline{Y}|=c$.
Due to the equation (\ref{cond}), we can integrate the class-prior information into the risk formulations of MCL and MCUL as follows:
\begin{theorem}(MCL$^{cl}$ risk and MCUL$^{cl}$ risk)
\label{CPI}
\begin{align}
\label{tot2}&R(\bm{g})=\mathbb{E}_{\overline{p}(\bm{x},\overline{Y})}\left[\mathcal{L}(\bm{g}(\bm{x}))-\lambda\sum\limits_{y\in\overline{Y}}\ell(\bm{g}(\bm{x}),y)\right]\\
\label{tot3}&=\mathbb{E}_{\overline{p}(\bm{x},\overline{Y})}\left[(1-\gamma)\mathcal{L}(\bm{g}(\bm{x}))-\lambda\sum\limits_{y\in\overline{Y}}\ell(\bm{g}(\bm{x}),y)\right]+\gamma\mathbb{E}_{\overline{p}(\bm{x})}\left[\mathcal{L}(\bm{g}(\bm{x}))\right]
\end{align}
where $\lambda=({K-1})/{\sum\limits_{c=1}^{K-1}c\pi_{c}}$ and $\gamma\in[0,1]$ is a trade-off parameter. (\ref{tot2}) and (\ref{tot3}) are called MCL$^{cl}$ risk and MCUL$^{cl}$ risk respectively.
\end{theorem}
The proof can be found in the Appendix \ref{AC}.

From the Theorem \ref{CPI}, the coefficient $\bm{\alpha}$ in (\ref{tot1}) and (\ref{MCUL}) is substituted by the class-prior $\pi_{c}$. Compared with (\ref{tot1}) that converges in the rate of $\mathcal{O}(\sum_{c=1}^{K-1}1/\sqrt{n_{c}})$, MCL$^{cl}$ risk (\ref{tot2}) converges in $\mathcal{O}(1/\sqrt{n})$, which often indicates a faster convergence rate. We will experimentally evaluate the helpfulness of integrating class-prior information in the next section.
\section{Experiments}

In this section, we experimentally evaluate the proposed methods on nine benchmark datasets including: PENDIGITS, LETTER, SATIMAGE, USPS, MNIST \cite{MNIST}, Fashion-MNIST \cite{FMNIST}, Kuzushi-MNIST \cite{kmnist}, EMNIST-balanced \cite{EMNIST} and SVHN \cite{SVHN}. The first three datasets can be downloaded from the \textit{UCI machine learning repository} and all the other datasets are available on public websites. We compare three complementary-label learning baseline methods: Pairwise Comparison(\textit{PC}) with sigmoid loss from \cite{DBLP:conf/nips/IshidaNHS17}, Forward Correction(\textit{Fwd}) from \cite{DBLP:conf/eccv/YuLGT18} and Gradient Ascent(\textit{GA}) from \cite{DBLP:conf/icml/IshidaNMS19}.

The details of the datasets are shown in the following sections. The implementation is based on Pytorch.
\subsection{Experimental Setup}
In the experiments, the empirical risk minimization of MCL and MCUL is conducted by minimizing the risk formulation (\ref{unbiased}) and (\ref{EMMCUL}) \textit{w.r.t.} softmax cross-entropy loss. \textit{GA}, and \textit{Fwd} follow the setting above and \textit{PC} is trained with pairwise-comparison loss. \textit{Adam} \cite{Adam} is applied for optimization. All the datasets are split into training/testing sets with a 9:1 ratio and the training sets are further divided into training/validation sets with the same ratio.

To ensure that the assumption (\ref{E3}) is satisfied, each complementary-label set of size \textit{c} is generated by randomly choosing \textit{c} labels from the candidate labels other than the true label. Though a sample with \textit{c} complementary labels can  be simply split into \textit{c} samples with one complementary label each, it's obvious that the \textit{c} samples are not independent of each other. Therefore this approach will lead to serious violation of the fundamental \textit{i.i.d.} assumption. For fair comparison in these experiments, the complementary labels are generated in the same way as in \cite{DBLP:conf/icml/IshidaNMS19}.

For PENDIGITS, LETTER, SATIMAGE, USPS, and MNIST, a linear-in-input model with a bias term is used. For MNIST, the learning rate is fixed to 1e-4; weight decay 1e-4; maximum iterations 60000; and batch size is set to 100. For the rest datasets, the learning rate is selected from $\{$1e-1, 1e-2, 1e-3, 1e-4$\}$ and the number of maximum iterations is changed to 5000.

For Fashion-MNIST, Kuzushi-MNIST and EMNIST-balanced, a MLP model(\textit{d}-500-\textit{K}) is trained for 300 epochs. The learning rate and weight decay are fixed to 1e-4 and the batch size is 256. For SVHN, Resnet-18 \cite{resnet} is deployed and trained for 120 epochs. The learning rate is selected from $\{$1e-2, 1e-3, 1e-4$\}$ and weight decay is fixed to 5e-4. To alleviate overfitting by forcing the non-negativity of loss functions \cite{nn}, in the experiments on flexible models, MCL and MCUL losses are replaced by their absolute values.

Experiments on datasets with or without unlabeled samples are both conducted. When the unlabeled samples are incorporated, we randomly set 99\% of training samples to be unlabeled for datasets with less than 50000 samples, which is a common setting in previous studies of weakly-supervised learning \cite{DBLP:conf/icml/SakaiPNS17,nn}. The fraction is further increased to 99.5\% for datasets with more than 50000 samples.

In respect of parameter setting, the setting of baseline methods follow the previous work \cite{DBLP:conf/icml/IshidaNMS19}. For MCUL, we use $\gamma\!=0.1$. The parameter $\bm{\alpha}$ is set according to the equations below:
$$\begin{cases}
\alpha_{i}:\alpha_{j}= \frac{n_{i}}{(K-i)^{2}}:\frac{n_{j}}{(K-j)^{2}}.\\
\sum_{i=1}^{K-1}\alpha_{1} =1.
\end{cases}$$

In our methods, we make no assumption on the distribution of the size of complementary-label sets. To generate the multi-complementarily labeled samples as close to the reality, we suppose that samples with too few or too many complementary labels are less likely to appear. Then the $n_{i}$ follows the equation below:
$$n_{i}:n_{j}=e^{-(i-\mu)^{2}}:e^{-(j-\mu)^{2}}.$$ In the experiments, $\mu = \frac{K}{2}$ is used.
\subsection{Experiments on Linear Model and Flexible Models}
The experimental results of linear and flexible models are summarized in Table \ref{TB1} and Table \ref{TB22} respectively. The experimental results under the presence of unlabeled samples are shown in the second row corresponding to each dataset.

\begin{table*}[htbp]
\caption{\footnotesize Test mean and standard deviation of the classification accuracy of linear model for 10 trials. The best one is emphasized in bold. \#n, \#f and \#c denote the number of samples, features and classes of each dataset. The results of experiments under the presence of unlabeled samples are shown in the second row corresponding to each dataset.}
\footnotesize
\centering
\begin{tabular}{ccccccccc}\hline
\label{TB1}
Datasets&\#n&\#f&\#c&\textit{PC}&\textit{Fwd}&\textit{GA}&\textit{MCL}&\textit{MCUL}\\\hline
\rule{0pt}{10pt}
\multirow{2}*{PENDIGITS}&\multirow{2}*{10092}&\multirow{2}*{16}&\multirow{2}*{10}&62.98$\pm$5.09&77.91$\pm$2.83&15.01$\pm$6.71&\textbf{84.32$\pm$1.09}&--~~$\pm$~~--\\&&&&5.98$\pm$4.24&8.09$\pm$1.71&11.05$\pm$5.05&11.44$\pm$5.68&\textbf{28.03$\pm$8.10}\\\rule{0pt}{15pt}
\multirow{2}*{LETTER}&\multirow{2}*{20000}&\multirow{2}*{16}&\multirow{2}*{26}&9.17$\pm$2.44&9.37$\pm$2.52&4.68$\pm$0.81&\textbf{45.75$\pm$2.31}&--~~$\pm$~~--\\&&&&4.66$\pm$2.78&3.99$\pm$0.88&5.12$\pm$2.01&5.33$\pm$0.76&\textbf{8.41$\pm$1.86}\\\rule{0pt}{15pt}
\multirow{2}*{SATIMAGE}&\multirow{2}*{6435}&\multirow{2}*{36}&\multirow{2}*{6}&74.05$\pm$4.65&77.45$\pm$5.43&38.32$\pm$1.82&\textbf{82.10$\pm$1.82}&--~~$\pm$~~--\\&&&&15.92$\pm$6.67&17.77$\pm$10.98&17.68$\pm$12.08&21.13$\pm$7.57&\textbf{51.35$\pm$7.94}\\\rule{0pt}{15pt}
\multirow{2}*{USPS}&\multirow{2}*{9298}&\multirow{2}*{256}&\multirow{2}*{10}&41.75$\pm$5.45&46.15$\pm$13.10&14.19$\pm$6.07&\textbf{82.31$\pm$2.32}&--~~$\pm$~~--\\&&&&9.66$\pm$4.72&11.56$\pm$7.23&8.57$\pm$3.38&9.66$\pm$6.62&\textbf{26.20$\pm$3.98}\\\rule{0pt}{15pt}
\multirow{2}*{MNIST}&\multirow{2}*{70000}&\multirow{2}*{784}&\multirow{2}*{10}&51.21$\pm$4.87&52.17$\pm$4.88&67.80$\pm$1.89&\textbf{77.36$\pm$1.27}&--~~$\pm$~~--\\&&&&13.78$\pm$4.16&13.58$\pm$2.11&14.86$\pm$1.77&25.86$\pm$3.49&\textbf{38.21$\pm$5.24}\\\hline
\end{tabular}
\end{table*}

\begin{table*}[htbp]
\footnotesize
\caption{\footnotesize Test mean and standard deviation of the classification accuracy of flexible models for 4 trials. The best one is emphasized in bold. \#n, \#f and \#c denote the number of samples, features and classes of each dataset. The results of experiments under the presence of unlabeled samples are shown in the second row corresponding to each dataset.}
\centering
\begin{tabular}{ccccccccc}\hline
\label{TB22}
Datasets&\#n&\#f&\#c&\textit{PC}&\textit{Fwd}&\textit{GA}&\textit{MCL}&\textit{MCUL}\\\hline
\rule{0pt}{10pt}
\multirow{2}*{Fashion-MNIST}&\multirow{2}*{70000}&\multirow{2}*{784}&\multirow{2}*{10}&77.34$\pm$0.88&83.49$\pm$0.18&81.73$\pm$0.25&\textbf{84.97$\pm$0.09}&--~~$\pm$~~--\\&&&&25.75$\pm$4.39&23.36$\pm$3.05&22.34$\pm$3.18&49.81$\pm$4.42&\textbf{56.93$\pm$1.66}\\\rule{0pt}{15pt}
\multirow{2}*{Kuzushi-MNIST}&\multirow{2}*{70000}&\multirow{2}*{784}&\multirow{2}*{10}&59.31$\pm$1.07&66.46$\pm$0.17&70.68$\pm$0.88&\textbf{79.25$\pm$0.28}&--~~$\pm$~~--\\&&&&16.35$\pm$4.92&12.94$\pm$2.72&15.13$\pm$2.42&23.42$\pm$1.15&\textbf{31.04$\pm$1.89}\\\rule{0pt}{15pt}
\multirow{2}*{EMNIST-balanced}&\multirow{2}*{131600}&\multirow{2}*{784}&\multirow{2}*{47}&14.28$\pm$1.18&18.21$\pm$2.93&4.25$\pm$0.71&\textbf{65.41$\pm$0.22}&--~~$\pm$~~--\\&&&&2.36$\pm$0.36&2.68$\pm$0.08&2.54$\pm$0.36&4.14$\pm$1.13&\textbf{6.77$\pm$0.51}\\\rule{0pt}{15pt}
\multirow{2}*{SVHN}&\multirow{2}*{99289}&\multirow{2}*{1024}&\multirow{2}*{10}&20.74$\pm$2.61&76.27$\pm$2.07&6.87$\pm$0.41&\textbf{83.27$\pm$0.55}&--~~$\pm$~~--\\&&&&14.77$\pm$1.27&12.95$\pm$0.75&17.56$\pm$1.54&19.05$\pm$0.44&\textbf{19.47$\pm$0.17}\\\hline\end{tabular}
\end{table*}

\begin{figure*}[htbp]
\subfigure[FMNIST]{\includegraphics[keepaspectratio,width=8.5cm]{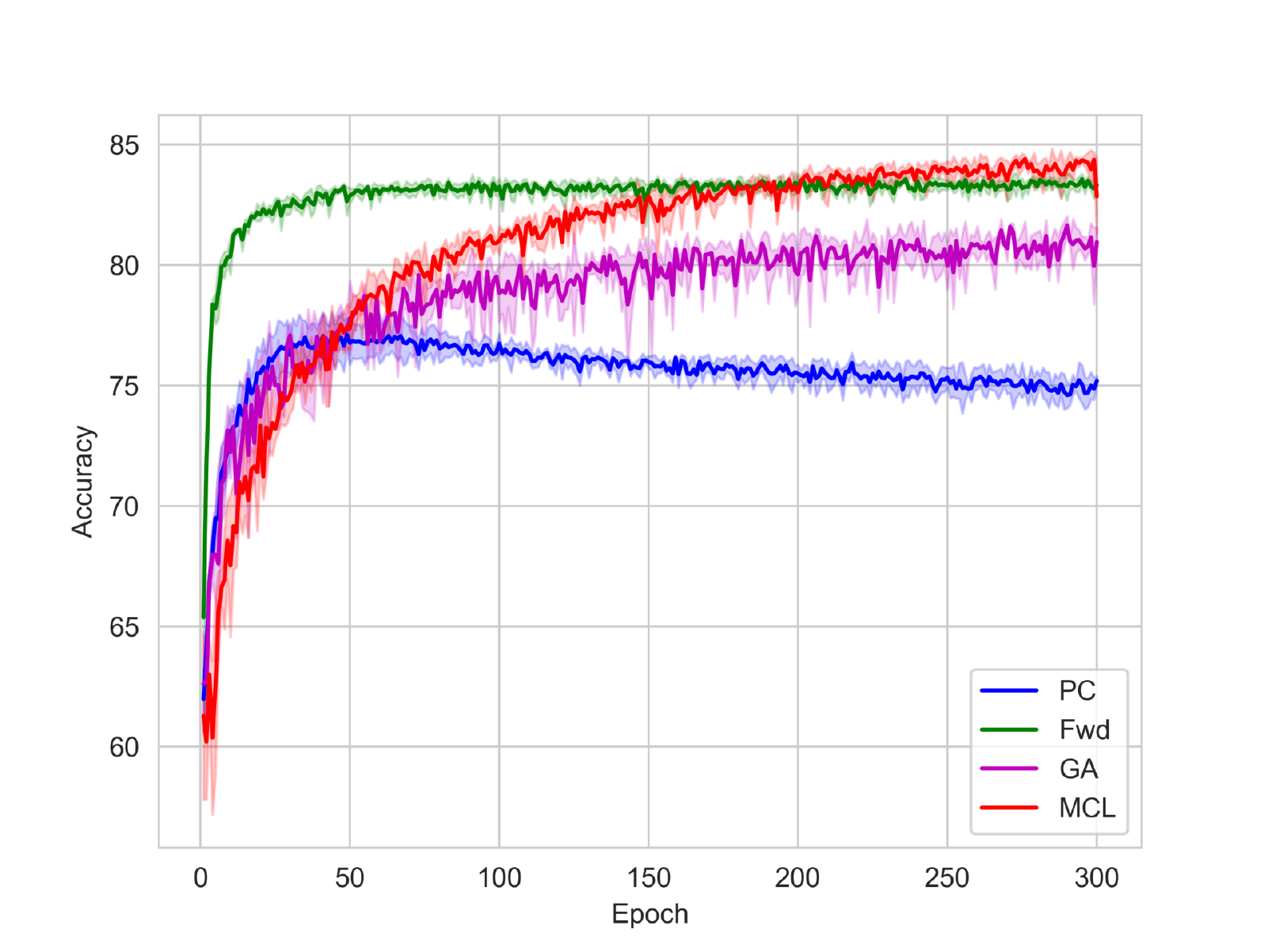}}
\subfigure[KMNIST]{\includegraphics[keepaspectratio,width=8.5cm]{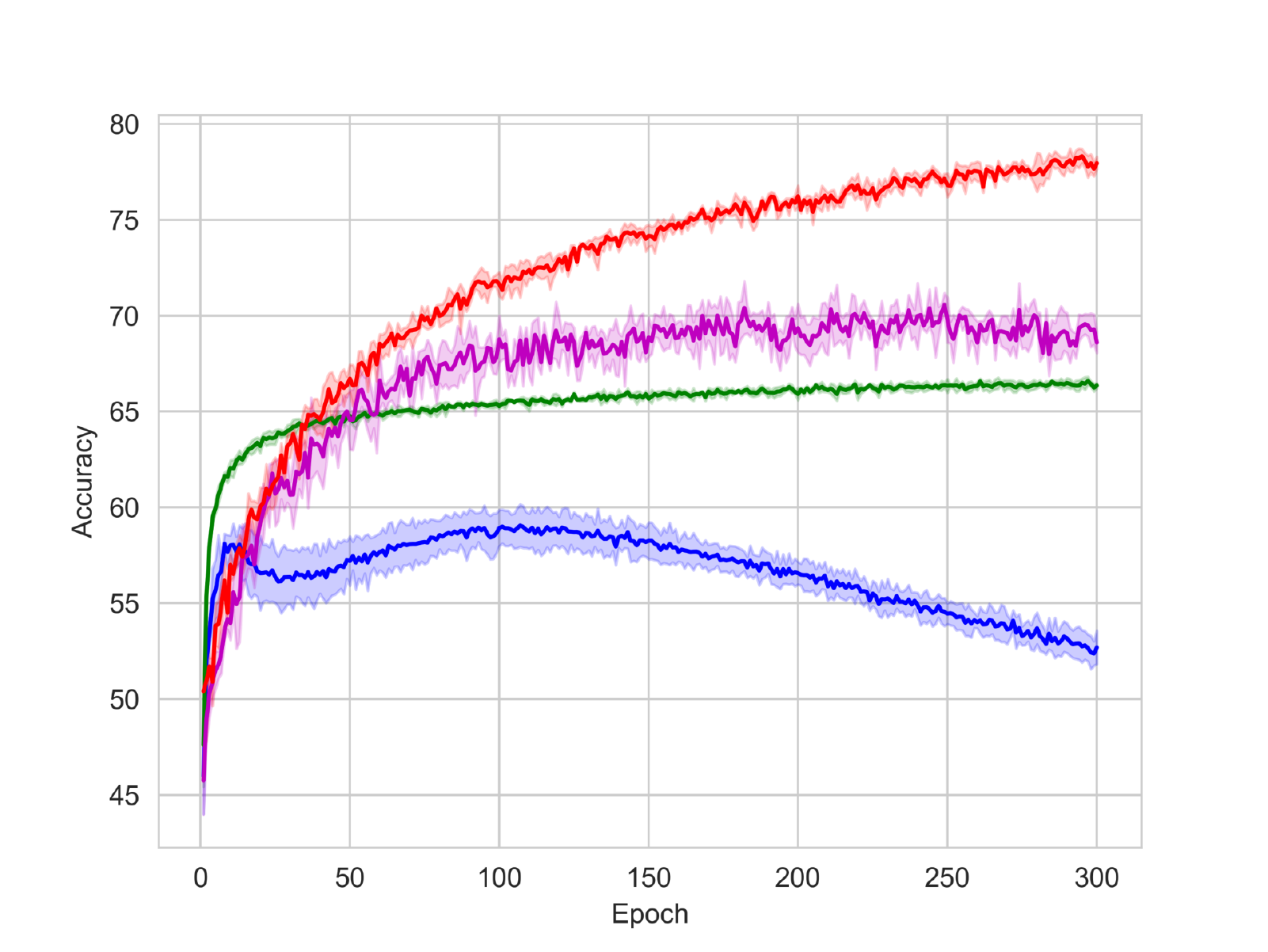}}\\
\subfigure[EMNIST-balanced]{\includegraphics[keepaspectratio,width=8.5cm]{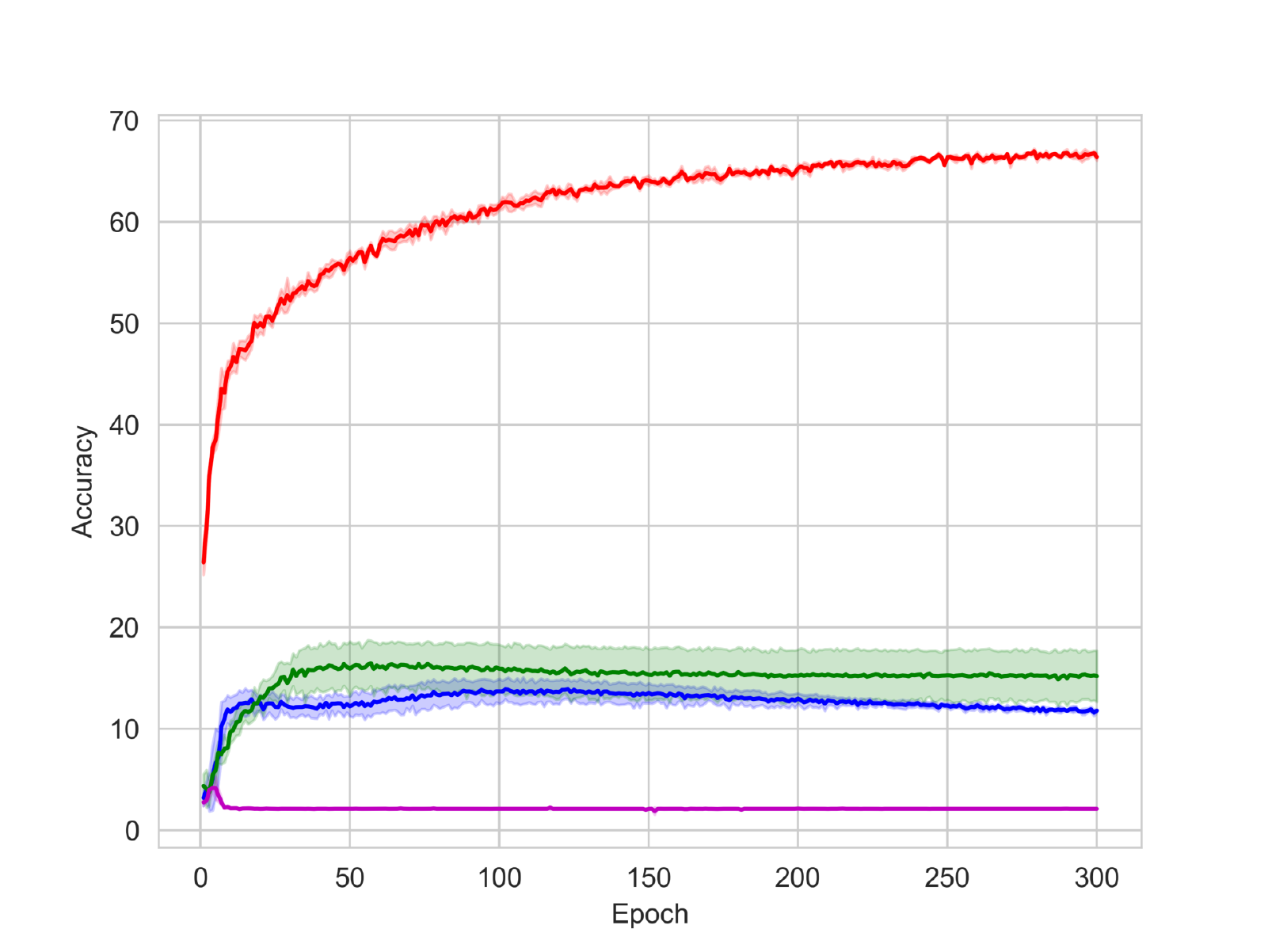}}
\subfigure[SVHN]{\includegraphics[keepaspectratio,width=8.5cm]{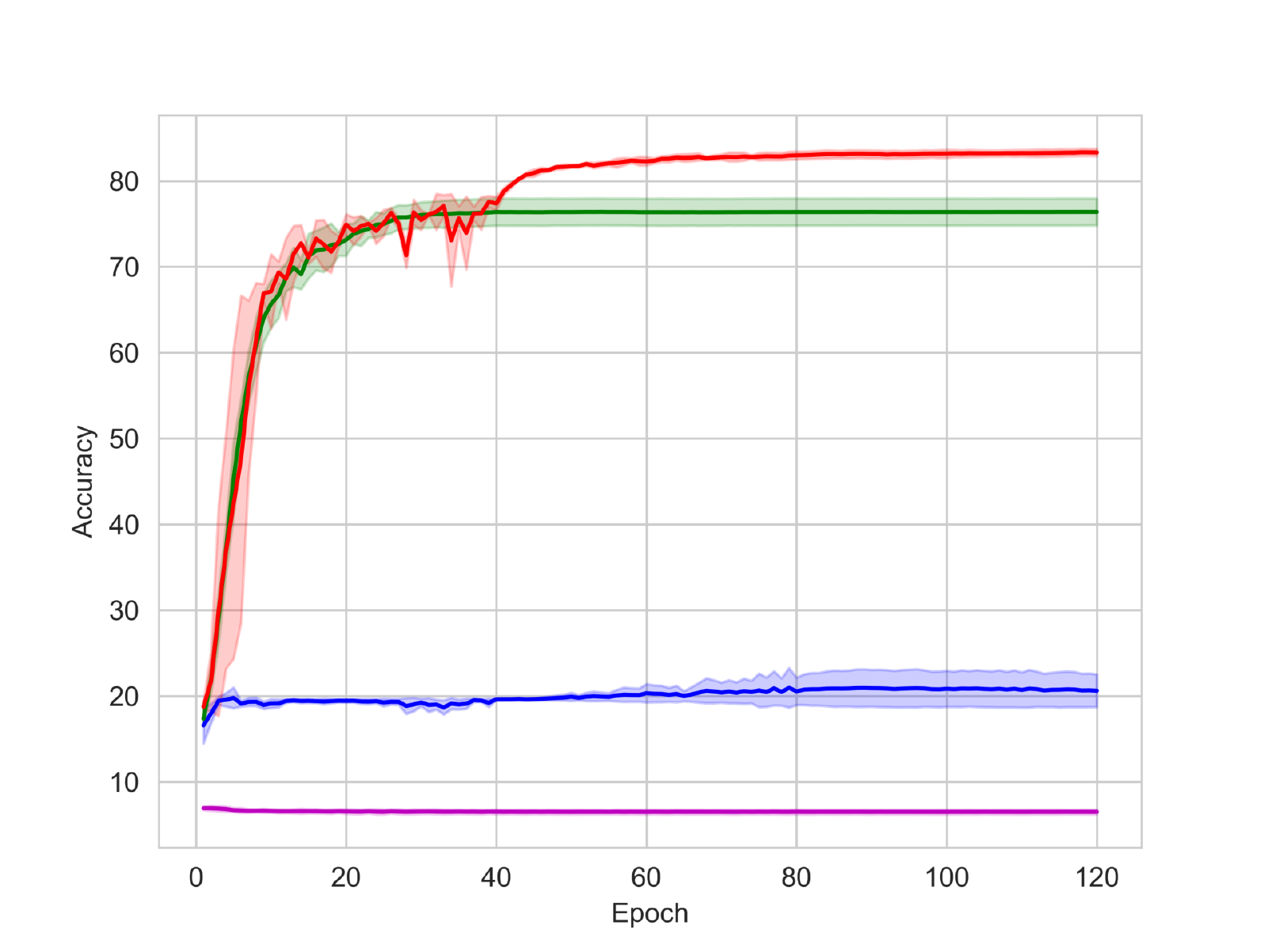}}
\caption{The training curves of baseline methods and MCL on flexible models.}

\label{F1}
\end{figure*}
{\em Results of MCL:} First we compare the proposed MCL framework with the three baseline methods. From the experimental results, we can see that MCL framework outperforms the baseline methods on all the datasets regardless of which model is applied. The superiority of MCL is especially apparent when the datasets have a large number of classes. Due to the experimental results on LETTER and EMNIST-balanced, it can be seen that the baseline methods can hardly generate effective classifiers. Furthermore, the training curves in Figure \ref{F1} show that in most cases, MCL can generate better classifiers and converge in a faster rate. Compared with the baseline methods, MCL remains valid owing to the capability of utilizing multi-complementarily labeled samples.

{\em Results of MCUL:} From Table \ref{T1}, as can be seen, under the presence of a great percentage of unlabeled samples, the baseline methods suffer from the lack of complementarily labeled samples, while MCUL can still enhance its performance by incorporating unlabeled samples. Moreover, in the cases that only a small number of complementarily labeled samples are available, the performance of baseline methods \textit{GA} is seriously degraded due to the imprecise estimation of class prior, which is one of the reasons that the performance of \textit{GA} is relatively weak with a small fraction of complementarily labeled samples.
\subsection{Experiments with Accessible Class-Prior Probability}
In this section, we further show the benefits of integrating the class-prior information. We compare the performance of MCL and MCUL with MCL$^{cl}$ and MCUL$^{cl}$ on linear model. The experimental results are summarized in Table \ref{TB2} and the setup is consistent with that in the previous experiments.

From the results, we can see that by integrating the class-prior probability, MCL$^{cl}$ outperforms MCL on most datasets and MCUL$^{cl}$ performs better than MCUL on all the datasets, which shows the helpfulness of utilizing the class-prior information when it is accessible. Furthermore, under the presence of a great percentage of unlabeled samples, the integration of class-prior information can always boost the performance of models. A rational explanation is that when supervision information is inadequate, the benefits of incorporating class-prior information is more conspicuous.

\begin{table}[tbph]
\caption{\footnotesize Test mean and standard deviation of the classification accuracy of linear model for 10 trials. The best one is emphasized in bold. The results of experiments under the presence of unlabeled samples are shown in the second row corresponding to each dataset.}
\footnotesize
\centering
\begin{tabular}{ccccc}\hline
\label{TB2}
Datasets&\textit{MCL}&\textit{MCUL}&\textit{MCL}$^{cl}$&\textit{MCUL}$^{cl}$\\\hline\rule{0pt}{10pt}
\multirow{2}*{PENDIGITS}&\textbf{84.32$\pm$1.09}&--~~$\pm$~~--&84.27$\pm$0.26&--~~$\pm$~~--\\&11.44$\pm$5.68&28.03$\pm$8.10&13.82$\pm$2.39&\textbf{38.47$\pm$5.68}\\\rule{0pt}{15pt}
\multirow{2}*{LETTER}&45.75$\pm$2.31&--~~$\pm$~~--&\textbf{53.29$\pm$5.31}&--~~$\pm$~~--\\&5.33$\pm$0.76&8.41$\pm$1.86&17.32$\pm$1.84&\textbf{18.54$\pm$0.92}\\\rule{0pt}{15pt}
\multirow{2}*{SATIMAGE}&82.10$\pm$1.82&--~~$\pm$~~--&\textbf{83.53$\pm$1.19}&--~~$\pm$~~--\\&21.13$\pm$7.57&51.35$\pm$7.94&39.61$\pm$9.57&\textbf{54.23$\pm$6.14}\\\rule{0pt}{15pt}
\multirow{2}*{USPS}&\textbf{82.31$\pm$2.32}&--~~$\pm$~~--&81.07$\pm$1.58&--~~$\pm$~~--\\&9.66$\pm$6.62&26.20$\pm$3.98&24.29$\pm$4.72&\textbf{39.58$\pm$2.99}\\\rule{0pt}{15pt}
\multirow{2}*{MNIST}&77.36$\pm$1.27&--~~$\pm$~~--&\textbf{81.47$\pm$2.46}&--~~$\pm$~~--\\&25.86$\pm$3.49&38.21$\pm$5.24&27.54$\pm$2.71&\textbf{41.58$\pm$5.27}\\\hline
\end{tabular}
\end{table}

\section{Conclusion}
We first derive the MCL framework to learn from samples with any number of complementary labels for arbitrary losses and models. Then we incorporate unlabeled data into the risk formulation and propose the MCUL framework to enhance the performance of MCL by learning from multi-complementarily labeled data and unlabeled data simultaneously. We further show that the estimation error bounds of the proposed methods are in the optimal parametric convergence rate. Finally, we conduct experiments and show our methods outperform the current state-of-the-art methods on both linear and deep model. A promising direction is applying our methods on crowdsourcing and other weakly-supervised classification scenarios, which is our future work.
\section*{Acknowledgement}
This work was supported in part by the National Natural Science Foundation of China under Grant 11671010, and Beijing Natural Science Foundation under Grant 4172035.

\newpage
\section*{Appendix}
\subsection*{A. Proof of Lemma \ref{T1}}
\label{AA}

\begin{proof}
Suppose $\overline{Y}\in\overline{\mathcal{Y}}_{c}$, then we can obtain $|\overline{Y}|=\bc$. Denote $\{\overline{Y}|~\overline{Y}\in\mathcal{Y}_{c},y\notin\overline{Y}\}$ with $\overline{\mathcal{Y}}^{y}_{c}$. Due to assumption (\ref{E3}), the equations below hold:
$$\begin{cases}
\sum\limits_{\overline{Y}\in\overline{\mathcal{Y}}^{y}_{c}}\overline{p}_{c}(\bm{x},\overline{Y})=p(\bm{x},y)+\frac{K-\bc-1}{K-1}\sum\limits_{\hat{y}\not=y}p(\bm{x},\hat{y}),\\
\sum\limits_{\overline{Y}\notin\overline{\mathcal{Y}}^{y}_{c}}\overline{p}_{c}(\bm{x},\overline{Y})=\frac{\bc}{K-1}\sum\limits_{\hat{y}\not=y}p(\bm{x},\hat{y}).\\
\end{cases}
$$
We can get the following equation by substituting the second equation above into the first equation:
\begin{align}
\label{p1}
    p(\bm{x},y)&=\sum\limits_{\overline{Y}\in\overline{\mathcal{Y}}^{y}_{c}}\overline{p}_{c}(\bm{x},\overline{Y})-\frac{K-\bc-1}{\bc}\sum\limits_{\overline{Y}\notin\overline{\mathcal{Y}}^{y}_{c}}\overline{p}_{c}(\bm{x},\overline{Y})\nonumber\\
    &=\sum\limits_{\overline{Y}\in\overline{\mathcal{Y}}_{c}}\overline{p}_{c}(\bm{x},\overline{Y})-\frac{K-1}{\bc}\sum\limits_{\overline{Y}\notin\overline{\mathcal{Y}}^{y}_{c}}\overline{p}_{c}(\bm{x},\overline{Y}).\nonumber
\end{align}
Then we can rewrite the classification risk (\ref{E1}) due to the equation above:
\begin{align}
    R(\bm{g})&= \sum\limits_{y}\int \ell(\bm{g}(\bm{x}),y)p(\bm{x},y)d\bm{x}\nonumber\\
    &=\sum\limits_{y}\int \ell(\bm{g}(\bm{x}),y)\sum\limits_{\overline{Y}\in\overline{\mathcal{Y}}_{c}}\overline{p}_{c}(\bm{x},\overline{Y})d\bm{x}\nonumber\\
    &~~~~-\frac{K-1}{\bc}\sum\limits_{y}\int \ell(\bm{g}(\bm{x}),y)\sum\limits_{\overline{Y}\notin\overline{\mathcal{Y}}^{y}_{c}}\overline{p}_{c}(\bm{x},\overline{Y})d\bm{x}.\nonumber
\end{align}
By exchanging the order of summation, we can get an equal version of the last equation above and finally conclude the proof:
\begin{align}
    R(\bm{g})&=\sum\limits_{\overline{Y}\in\overline{\mathcal{Y}}_{c}}\int \sum\limits_{y}\ell(\bm{g}(\bm{x}),y)\overline{p}_{c}(\bm{x},\overline{Y})d\bm{x}\nonumber\\
    &~~~~-\frac{K-1}{\bc}\sum\limits_{\overline{Y}\in\overline{\mathcal{Y}}_{c}}\int\sum\limits_{y\in\overline{Y}}\ell(\bm{g}(\bm{x}),y)\overline{p}_{c}(\bm{x},\overline{Y})d\bm{x}\nonumber\\
    &=\sum\limits_{\overline{Y}\in\overline{\mathcal{Y}}_{c}}\int\left(\sum\limits_{y}\ell(\bm{g}(\bm{x}),y)-\dfrac{K\!-\!1}{|\overline{Y}|}\sum\limits_{y\in\overline{Y}}\ell(\bm{g}(\bm{x}),y)\right)\overline{p}_{c}(\bm{x},\overline{Y})d\bm{x}\nonumber\\
    &=\mathbb{E}_{\overline{p}(\bm{x},\overline{Y})}[\overline{l}(\bm{x},\overline{Y})]=R_{c}(\bm{g}).\nonumber
\end{align}
\end{proof}

\subsection*{B. Proof of Theorem \ref{MCUL_bound}}
\label{AB}
The proof of Theorem \ref{MCL_bound} is omitted since it is the special case of Theorem \ref{MCUL_bound} by setting $\gamma$ to 0. 

First we introduce the Talagrand's contraction lemma \cite{tala}:
\begin{lemma}
\label{L7}
Let $\mathcal{G}$ be a class of real functions and $\mathcal{G}^{K}=[\mathcal{G}_{i}]_{i=1}^{K}$ be a K-dimensional function class and $\ell:\mathbb{R}^{K}\rightarrow \mathbb{R}$ a Lipschitz function with constant $L_{\ell}$ and  $\ell(0)=0$. Then $\mathfrak{R}_{n}(\ell\circ\mathcal{G}^{K})\leq KL_{\ell}\mathfrak{R}_{n}(\mathcal{G})$.
\end{lemma}

To apply Talagrand's contraction lemma, we use the shifted loss $\tilde{\ell}(z)=\ell(z)-\ell(0)$ instead of $\ell(z)$. Then we abbreviate some complex terms in these forms:
$$\begin{cases}
    \bm{A}=\gamma\mathbb{E}_{p(\bm{x})}[\sum\limits_{y}\tilde{\ell}(\bm{g}(\bm{x}),y)]\\ \bm{B}_{c}=\mathbb{E}_{\overline{p}_{c}(\bm{x},\overline{Y})}[(1-\gamma)\sum\limits_{y}\tilde{\ell}(\bm{g}(\bm{x}),y)-\dfrac{K\!-\!1}{c}\sum\limits_{y\in\overline{Y}}\tilde{\ell}(\bm{g}(\bm{x}),y)]\\
     \tilde{\ell}_{A}(\bm{g}(\bm{x}))=\gamma\sum\limits_{y}\tilde{\ell}(\bm{g}(\bm{x}),y)\\  \tilde{\ell}_{B}^{c}(\bm{g}(\bm{x}),\overline{Y})=(1-\gamma)\sum\limits_{y}\tilde{\ell}(\bm{g}(\bm{x}),y)-\dfrac{K\!-\!1}{c}\sum\limits_{y\in\overline{Y}}\tilde{\ell}(\bm{g}(\bm{x}),y).
\end{cases}$$

We give the following conclusions:
\begin{lemma} Denote $\overline{p}_{c}(\bm{x},\overline{Y})$ with $\overline{p}_{c}$:
\begin{align}
    &\mathfrak{R}_{n}(\tilde{\ell}_{A}\circ\mathcal{G}^{K})\leq \gamma K^{2}L_{\ell}\mathfrak{R}_{n}(\mathcal{G}),\nonumber\\
    &\mathfrak{R}_{n,\overline{p}_{c}}(\tilde{\ell}_{B}^{c}\circ\mathcal{G}^{K})\leq \left(\frac{K-1}{\bc}+1-\gamma\right)K^{2}L_{\ell}\mathfrak{R}_{n}(\mathcal{G}).\nonumber
\end{align}
\end{lemma}
\begin{proof}
The first inequality can be deduced from Lemma \ref{L7} directly. By definition and the sub-additivity of supremum: 
\begin{align}
    \mathfrak{R}_{n,\overline{p}_{c}}(\tilde{\ell}_{B}^{c}\circ\mathcal{G}^{K})&=\mathbb{E}_{\mathcal{S}_{c}}\mathbb{E}_{\bm{\sigma}}\left[\sup\limits_{\bm{g}\in\mathcal{G}^{K}}\frac{1}{n}\sum\limits_{(\bm{x}_{i},\overline{Y}_{i})\in\mathcal{S}_{c}}\sigma_{i}\tilde{\ell}_{B}(\bm{g}(\bm{x}_{i}),\overline{Y}_{i})\right]\nonumber\\
    &\leq(1-\gamma)\mathbb{E}_{\mathcal{S}_{c}}\mathbb{E}_{\bm{\sigma}}\left[\sup\limits_{\bm{g}\in\mathcal{G}^{K}}\frac{1}{n}\sum\limits_{(\bm{x}_{i},\overline{Y}_{i})\in\mathcal{S}_{c}}\sigma_{i}\sum\limits_{y}\tilde{\ell}(\bm{g}(\bm{x}_{i}),y)\right]\nonumber\\&~~~+\frac{K-1}{\bc}\mathbb{E}_{\mathcal{S}_{c}}\mathbb{E}_{\bm{\sigma}}\left[\sup\limits_{\bm{g}\in\mathcal{G}^{K}}\frac{1}{n}\sum\limits_{(\bm{x}_{i},\overline{Y}_{i})\in\mathcal{S}_{c}}\sigma_{i}\sum\limits_{y\in\overline{Y}_{i}}\tilde{\ell}(\bm{g}(\bm{x}_{i}),y)\right]\nonumber
\end{align}
The $\tilde{\ell}_{B}$ is a fixed loss function and the first equation holds. Since $\Sigma_{y}\tilde{l}$ is independent of $\overline{Y}_{i}$, we can get:
\begin{align}
    \mathbb{E}_{\mathcal{S}_{c}}\mathbb{E}_{\bm{\sigma}}\left[\sup\limits_{\bm{g}\in\mathcal{G}^{K}}\frac{1}{n}\sum\limits_{(\bm{x}_{i},\overline{Y}_{i})\in\mathcal{S}_{c}}\sigma_{i}\sum\limits_{y}\tilde{\ell}(\bm{g}(\bm{x}_{i}),y)\right]=\mathbb{E}_{\mathcal{X}}\mathbb{E}_{\bm{\sigma}}\left[\sup\limits_{\bm{g}\in\mathcal{G}^{K}}\frac{1}{n}\sum\limits_{\bm{x}_{i}\in\mathcal{X}}\sigma_{i}\sum\limits_{y}\tilde{\ell}(\bm{g}(\bm{x}_{i}),y)\right]\nonumber
\end{align}
Let $I(\cdot)$ be the indicator function and $\alpha_{i} = 2I(y\in\overline{Y}_{i})-1$. Then we have the conclusion below:
\begin{align}
    \mathbb{E}_{\mathcal{S}_{c}}\mathbb{E}_{\bm{\sigma}}\left[\sup\limits_{\bm{g}\in\mathcal{G}^{K}}\frac{1}{n}\sum\limits_{(\bm{x}_{i},\overline{Y}_{i})\in\mathcal{S}_{c}}\sigma_{i}\sum\limits_{y\in\overline{Y}_{i}}\tilde{\ell}(\bm{g}(\bm{x}_{i}),y)\right]&=\mathbb{E}_{\mathcal{S}_{c}}\mathbb{E}_{\bm{\sigma}}\left[\sup\limits_{\bm{g}\in\mathcal{G}^{K}}\frac{1}{2n}\sum\limits_{(\bm{x}_{i},\overline{Y}_{i})\in\mathcal{S}_{c}}\sigma_{i}\sum\limits_{y}\tilde{\ell}(\bm{g}(\bm{x}_{i}),y)(\alpha_{i}+1)\right]\nonumber\\&\leq\mathbb{E}_{\mathcal{S}_{c}}\mathbb{E}_{\bm{\sigma}}\left[\sup\limits_{\bm{g}\in\mathcal{G}^{K}}\frac{1}{2n}\sum\limits_{(\bm{x}_{i},\overline{Y}_{i})\in\mathcal{S}_{c}}\alpha_{i}\sigma_{i}\sum\limits_{y}\tilde{\ell}(\bm{g}(\bm{x}_{i}),y)\right]\nonumber\\&+\mathbb{E}_{\mathcal{S}_{c}}\mathbb{E}_{\bm{\sigma}}\left[\sup\limits_{\bm{g}\in\mathcal{G}^{K}}\frac{1}{2n}\sum\limits_{(\bm{x}_{i},\overline{Y}_{i})\in\mathcal{S}_{c}}\sigma_{i}\sum\limits_{y}\tilde{\ell}(\bm{g}(\bm{x}_{i}),y)\right]\nonumber\\
    &=\mathbb{E}_{\mathcal{X}}\mathbb{E}_{\bm{\sigma}}\left[\sup\limits_{\bm{g}\in\mathcal{G}^{K}}\frac{1}{n}\sum\limits_{\bm{x}_{i}\in\mathcal{X}}\sigma_{i}\sum\limits_{y}\tilde{\ell}(\bm{g}(\bm{x}_{i}),y)\right]\nonumber
\end{align}
Then the inequalities hold:
\begin{align}
    \mathfrak{R}_{n,\overline{p}_{c}}(\tilde{\ell}_{B}^{c}\circ\mathcal{G}^{K})&\leq(1-\gamma)\mathbb{E}_{\mathcal{X}}\mathbb{E}_{\bm{\sigma}}\left[\sup\limits_{\bm{g}\in\mathcal{G}^{K}}\frac{1}{n}\sum\limits_{\bm{x}_{i}\in\mathcal{X}}\sigma_{i}\sum\limits_{y}\tilde{\ell}(\bm{g}(\bm{x}_{i}),y)\right]\nonumber\\&~~~+\frac{(K-1)}{\bc}\mathbb{E}_{\mathcal{X}}\mathbb{E}_{\bm{\sigma}}\left[\sup\limits_{\bm{g}\in\mathcal{G}^{K}}\frac{1}{n}\sum\limits_{\bm{x}_{i}\in\mathcal{X}}\sigma_{i}\sum\limits_{y}\tilde{\ell}(\bm{g}(\bm{x}_{i}),y)\right]\nonumber\\
    &=\left(\frac{(K-1}{\gamma\bc}+\frac{(1-\gamma)}{\gamma}\right)\mathfrak{R}_{n}(\tilde{\ell}_{A}\circ\mathcal{G}^{K})\nonumber\\
    &\leq\left(\frac{K-1}{\bc}+1-\gamma\right)K^{2}L_{\ell}\mathfrak{R}_{n}(\mathcal{G})\nonumber
\end{align}
\end{proof}

We can bound $\mbox{sup}_{\bm{g}\in\mathcal{G}^{K}}\left|\bm{A}-\hat{\bm{A}}\right|$ and $\mbox{sup}_{\bm{g}\in\mathcal{G}^{K}}\left|\bm{B}_{c}-\hat{\bm{B}}_{c}\right|$ using Mcdiarmid's inequality \cite{Mc}:
\begin{lemma}\label{L9}For a certain c, the inequalities below hold with probability at least $1-\delta$:
\begin{align}
    &\sup\limits_{\bm{g}\in\mathcal{G}^{K}}\left|\bm{A}-\hat{\bm{A}}\right|\leq \gamma K\left(2KL_{\ell}\mathfrak{R}_{n_{u}}(\mathcal{G})+C_{\ell}\sqrt{\dfrac{\ln2/\delta}{2n_{u}}}\right)\nonumber\\
    &\sup\limits_{\bm{g}\in\mathcal{G}^{K}}\left|\bm{B}_{c}-\hat{\bm{B}}_{c}\right|\leq \left(\frac{K-1}{\bc}+1-\gamma\right)K\left(2KL_{\ell}\mathfrak{R}_{n_{c}}(\mathcal{G})+C_{\ell}\sqrt{\dfrac{\ln2/\delta}{2n_{c}}}\right).\nonumber
\end{align}
\end{lemma}
\begin{proof}
We are going to prove the first inequality and the second can be proved in a similar way. Firstly, we consider the single direction $\mbox{sup}_{\bm{g}\in\mathcal{G}^{K}}(\bm{A}-\hat{\bm{A}})$. The $\hat{\ell}_{A}$ will not exceed $\gamma KC_{\ell}$ due to the definition, then
the change of $\mbox{sup}_{\bm{g}\in\mathcal{G}^{K}}(\bm{A}-\hat{\bm{A}})$ will not exceed $\gamma KC_{\ell}/n$ when we replace a single $\bm{x}_{i}$ with $\bm{x}_{i}^{'}$. Due to the Mcdiarmid's inequality, the inequality below holds with probability at least $1-\delta/2$:
\begin{align}
    \sup\limits_{\bm{g}\in\mathcal{G}^{K}}(\bm{A}-\hat{\bm{A}})\leq \mathbb{E}\left[\sup\limits_{\bm{g}\in\mathcal{G}^{K}}(\bm{A}-\hat{\bm{A}})\right]+\gamma KC_{\ell}\sqrt{\dfrac{\ln2/\delta}{2n_{u}}}\nonumber
\end{align}
Due to the symmetrization inequality \cite{foundation}, we can obtain that:
\begin{align}
\mathbb{E}\left[\sup\limits_{\bm{g}\in\mathcal{G}^{K}}(\bm{A}-\hat{\bm{A}})\right]&\leq 2\mathfrak{R}_{n}(\ell_{A}\circ\mathcal{G}^{K})\nonumber\\&\leq 2\gamma K^{2}L_{\ell}\mathfrak{R}_{n}(\mathcal{G})\nonumber
\end{align}
The other direction is similar.
\end{proof}
~\\

Now we can prove the Theorem \ref{MCUL_bound}:
\begin{proof}
Notice that $\hat{R}_{\mbox{\tiny{MCUL}}}(\bm{g}_{\mbox{\tiny{MCUL}}})\leq \hat{R}_{\mbox{\tiny{MCUL}}}(\bm{g}^{*})$. Due to Lemma \ref{T2}, we can get:
\begin{align}
    R(\hat{\bm{g}}_{\mbox{\tiny{MCUL}}})-R(\bm{g}^{*})&=R_{\mbox{\tiny{MCUL}}}(\hat{\bm{g}}_{\mbox{\tiny{MCUL}}})-R_{\mbox{\tiny{MCUL}}}(\bm{g}^{*})\nonumber\\
    &=R_{\mbox{\tiny{MCUL}}}(\hat{\bm{g}}_{\mbox{\tiny{MCUL}}})-\hat{R}_{\mbox{\tiny{MCUL}}}(\hat{\bm{g}}_{\mbox{\tiny{MCUL}}})+\hat{R}_{\mbox{\tiny{MCUL}}}(\hat{\bm{g}}_{\mbox{\tiny{MCUL}}})\nonumber\\&~~~-\hat{R}_{\mbox{\tiny{MCUL}}}(\bm{g}^{*})+\hat{R}_{\mbox{\tiny{MCUL}}}(\bm{g}^{*})-R_{\mbox{\tiny{MCUL}}}(\bm{g}^{*})\nonumber\\
    &\leq R_{\mbox{\tiny{MCUL}}}(\hat{\bm{g}}_{\mbox{\tiny{MCUL}}})-\hat{R}_{\mbox{\tiny{MCUL}}}(\hat{\bm{g}}_{\mbox{\tiny{MCUL}}})+\hat{R}_{\mbox{\tiny{MCUL}}}(\bm{g}^{*})-R_{\mbox{\tiny{MCUL}}}(\bm{g}^{*})\nonumber\\
    &\leq 2\sup_{\bm{g}\in \mathcal{G}^{K}}\left|R_{\mbox{\tiny{MCUL}}}(\bm{g})-\hat{R}_{\mbox{\tiny{MCUL}}}(\bm{g})\right|.\nonumber
\end{align}
 According to the sub-additivity of supremum, we can get the inequality below:
\begin{align}
\sup_{\bm{g}\in \mathcal{G}^{K}}\left|R_{\mbox{\tiny{MCUL}}}(\bm{g})-\hat{R}_{\mbox{\tiny{MCUL}}}(\bm{g})\right|&\leq \sum\limits_{c=1}^{K-1}\alpha_{c}\sup\limits_{\bm{g}\in \mathcal{G}^{K}}\left|R_{c}^{u}(\bm{g})-\hat{R}_{c}^{u}(\bm{g})\right|\nonumber\\
&\leq \sum\limits_{c=1}^{K-1}\alpha_{c}\left(\sup\limits_{\bm{g}\in\mathcal{G}^{K}}\left|\bm{A}-\hat{\bm{A}}\right|+\sup\limits_{\bm{g}\in\mathcal{G}^{K}}\left|\bm{B}_{c}-\hat{\bm{B}}_{c}\right|\right)\nonumber\\
&=\sup\limits_{\bm{g}\in\mathcal{G}^{K}}\left|\bm{A}-\hat{\bm{A}}\right|+\sum\limits_{c=1}^{K-1}\alpha_{c}\sup\limits_{\bm{g}\in\mathcal{G}^{K}}\left|\bm{B}_{c}-\hat{\bm{B}}_{c}\right|\nonumber
\end{align}

Due to the union bound and Lemma \ref{L9}, the inequality below holds with probability at least $1-\delta$: 
\begin{align}
R(\hat{\bm{g}}_{\mbox{\tiny{MCUL}}})-R(\bm{g}^{*})&\leq2\sup_{\bm{g}\in \mathcal{G}^{K}}\left|R_{\mbox{\tiny{MCUL}}}(\bm{g})-\hat{R}_{\mbox{\tiny{MCUL}}}(\bm{g})\right|\nonumber\\&\leq2\left(\sup\limits_{\bm{g}\in\mathcal{G}^{K}}\left|\bm{A}-\hat{\bm{A}}\right|+\sum\limits_{c=1}^{K-1}\alpha_{c}\sup\limits_{\bm{g}\in\mathcal{G}^{K}}\left|\bm{B}_{c}-\hat{\bm{B}}_{c}\right|\right)\nonumber\\&\leq 2K\left[\gamma \left(2KL_{\ell}\mathfrak{R}_{n_{u}}(\mathcal{G})+C_{\ell}\sqrt{\dfrac{\ln(2K/\delta)}{2n_{u}}}\right)\right.\nonumber\\&\left.+\left(\frac{K-1}{\bc}+1-\gamma\right)\sum\limits_{c=1}^{K-1}\alpha_{c}\!\left(2KL_{\ell}\mathfrak{R}_{n_{c}}(\mathcal{G})+C_{\ell}\sqrt{\dfrac{\ln(2K/\delta)}{2n_{c}}}\right)\right].\nonumber
\end{align}
which concludes the proof.
\end{proof}

\subsection*{{Proof of Theorem \ref{CPI}}}
\label{AC}
\begin{proof}
\rm Denote $\{\overline{Y}|\overline{Y}\!\in\!\overline{\mathcal{Y}},y\not\in\overline{Y}\}$ with $\overline{\mathcal{Y}}^{y}$. Due to the equation (\ref{cond}), the equations following hold:
$$\begin{cases}
\sum\limits_{\overline{Y}\in\overline{\mathcal{Y}}^{y}}\overline{p}(\bm{x},\overline{Y})=p(\bm{x},y)+(1-\frac{\sum c\pi_{c}}{K-1})\sum\limits_{\hat{y}\not=y}p(\bm{x},\hat{y}).\nonumber\\
\sum\limits_{\overline{Y}\not\in\overline{\mathcal{Y}}^{y}}\overline{p}(\bm{x},\overline{Y})=\frac{\sum c\pi_{c}}{K-1}\sum\limits_{\hat{y}\not=y}p(\bm{x},\hat{y}).\nonumber
\end{cases}$$
By substituting the second equation into the first one, we can get:
\begin{align}
p(\bm{x},y)=\sum\limits_{\overline{Y}\in\overline{\mathcal{Y}}}\overline{p}(\bm{x},\overline{Y})-\frac{K-1}{\sum c\pi_{c}}\sum\limits_{\overline{Y}\not\in\overline{\mathcal{Y}}^{y}}\overline{p}(\bm{x},\overline{Y}).
\end{align}
By denoting $\frac{K-1}{\sum c\pi_{c}}$ with $\lambda$, the following equations hold:
\begin{align}
&R(\bm{g})=\mathbb{E}_{p(\bm{x},y)}\left[\ell(\bm{g}(\bm{x}),y)\right]=\sum\limits_{y=1}^{K}\int\ell(\bm{g}(\bm{x}),y)p(\bm{x},y)d\bm{x}\nonumber\\
&=\sum\limits_{y=1}^{K}\int\ell(\bm{g}(\bm{x}),y)\left(\!\sum\limits_{\overline{Y}\in\overline{\mathcal{Y}}}\overline{p}(\bm{x},\overline{Y})-\lambda\sum\limits_{\overline{Y}\not\in\overline{\mathcal{Y}}^{y}}\overline{p}(\bm{x},\overline{Y})\!\right)\!d\bm{x}\nonumber\\
&=\sum\limits_{\overline{Y}\in\overline{\mathcal{Y}}}\int\left(\sum\limits_{y=1}^{K}\ell(\bm{g}(\bm{x}),y)-\lambda\sum\limits_{y\in\overline{Y}}\ell(\bm{g}(\bm{x}),y)\right)\overline{p}(\bm{x},\overline{Y})d\bm{x}.\nonumber\\
&=\mathbb{E}_{\overline{p}(\bm{x},\overline{Y})}\left[\sum\limits_{y=1}^{K}\ell(\bm{g}(\bm{x}),y)-\lambda\sum\limits_{y\in\overline{Y}}\ell(\bm{g}(\bm{x}),y)\right].\nonumber\\
&=\mathbb{E}_{\overline{p}(\bm{x},\overline{Y})}\left[\mathcal{L}(\bm{g}(\bm{x}),y)-\lambda\sum\limits_{y\in\overline{Y}}\ell(\bm{g}(\bm{x}),y)\right].
\end{align}
which concludes the proof of equation (\ref{tot2}). (\ref{tot3}) can be proved in the same way as in the proof of Lemma \ref{Eu}.
\end{proof}
\end{document}